\documentclass[oneside,11pt]{article}
\pdfoutput=1 
\usepackage{arxiv}
\usepackage[utf8]{inputenc} 
\usepackage[T1]{fontenc}    
\usepackage{hyperref}       
\usepackage{url}            
\usepackage{booktabs}       

\usepackage{amsmath}
\usepackage{amsthm}
\usepackage{amsfonts}
\usepackage{amssymb}
\usepackage{graphicx}
\usepackage{enumerate}
\usepackage{mathrsfs}
\usepackage{braket}
\usepackage{nccmath}
\usepackage{nicefrac}       
\usepackage{microtype}      

\newtheorem{theorem}{Theorem}

\newtheorem{corollary}[theorem]{Corollary}

\newtheorem{proposition}[theorem]{Proposition}

\newcommand{\mbf}[1]{\mathbf{#1}}

\newcommand{\pow}[1]{\text{\bf pow}}

\title{On the generalization of Tanimoto-type kernels to real valued
functions}

\author{
  Sandor Szedmak \\ 
  Department of Computer Science\\
  Aalto University\\
  Espoo, Finland \\
 \texttt{sandor.szedmak@aalto.fi} \\
  \And
  Eric Bach \\ 
  Department of Computer Science \\
  Aalto University \\
  Espoo, Finland \\
  \texttt{eric.bach@aalto.fi} \\
}

\begin{document}

\maketitle

\begin{abstract}
The Tanimoto kernel (Jaccard index) is a well known tool to describe
the similarity between sets of binary attributes. It has been extended to
the case when the attributes are nonnegative real values. This paper introduces
a more general Tanimoto kernel formulation which allows to measure the 
similarity of arbitrary real-valued functions. This extension is
constructed by unifying the representation of the attributes via
properly chosen sets. After deriving the general form of the kernel,
explicit feature representation is extracted from the kernel function,
and a simply way of including general kernels into the Tanimoto kernel
is shown. Finally, the kernel is also expressed as a  quotient of
piecewise linear functions, and a smooth approximation is provided.        
\end{abstract}


\section{Introduction}

In a broad range of machine learning, pattern recognition and data
mining problems the data sources are given as sets of simple, mostly
binary attributes. These sets are generally represented by indicator 
functions. An element of the feature set corresponding to an attribute is
given by a class, and that element expresses that the data object is a
member of that class or not. Using this data representation, we can derive
similarity or dissimilarity measures between objects. Several of those
similarity measures can be formulated as positive definite inner
products, thus they can serve as kernel functions of a Reproducing Kernel
Hilbert space, \cite{Shawe-Taylor2004}. A well known instance
of these type of measures is the Jaccard index, which was introduced by  
\cite{doi:10.1111/j.1469-8137.1912.tb05611.x}. That index is a
normalized measure, thus decreases the effect 
caused by the different sizes of the sets, which otherwise could distort
the comparison. The Jaccard index 
has been extended beyond the case of the indicator function represented
sets by allowing nonnegative, real-valued feature vectors. This extension is
generally referred to as weighted Jaccard index \cite{Ioffe36928}. 
When the Jaccard index is used as a kernel function then it is generally
called as Tanimoto kernel honoring its first application by
\cite{nla.cat-vn1717218}. The kernel derived from the weighted Jaccard index is
also referred to as MinMax kernel, \cite{Ioffe36928}. 

These type of similarity measures, or kernels, are extensively used in
bio- and chemoinformatics. One of the most general application relates
to the representation of molecules which uses, so called, molecular
fingerprints. 
These fingerprints enumerate the occurrences of the characteristic
substructures in the molecules. We can mention here only a few but
illustrative examples of the applications: \cite{Hall1995}, \cite{butina_ci9803381}, \cite{doi:10.1002/9780470116449.ch6},
\cite{RALAIVOLA20051093}, \cite{Geppert_ci900419k}, \cite{10.5555/1893126}, \cite{LAVECCHIA2015318}
and \cite{Bach2018}.

The purpose of this paper is to extend the known Tanimoto kernels
(Jaccard indexes) to the case where the objects are represented by arbitrary
real-valued vectors or functions. We present the  derivation of the generalized
Tanimoto kernel and its relationship to the previously known
variants. Figure~\ref{fi_generalization_summary} illustrates the gradual   
extension of the application domain of the Tanimoto kernel.

\begin{figure}
  \begin{center}
    \includegraphics[width=3in,height=2.4in]{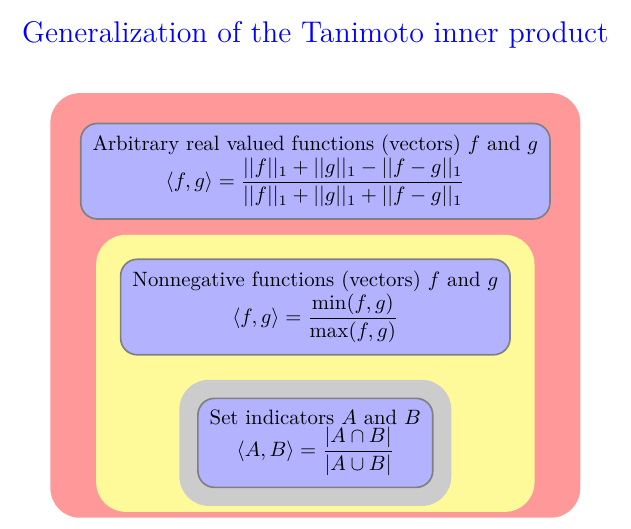}
  \end{center}
  \caption{The summary of the levels generalizing the inner product
  defining the Tanimoto kernel.}
\label{fi_generalization_summary}
\end{figure}

\subsection{Notations}
$\mathbb{R}_{+}$ denotes the set of nonnegative real numbers.  
Let $S$ be a set, and its power set is denoted by $2^{S}$. $|S|$ is 
the size of the set $S$. In this paper we assume that $|S|$ is finite. The inner
product between two functions $f,g:\mathcal{D}\mapsto\mathbb{R}$ with 
respect to a finite measure $\mu$ defined on the domain $\mathcal{D}$
is given by $\int \limits_{x\in \mathcal{D}} f(x)g(x) d\mu(x)$. In the sequel
we assume that for every function $f$ mentioned $|f|$ is bounded. The
finite dimensional vectors are taken as piecewise constant functions.      
   
The $L_1$ norm of a function $f:\mathcal{D}\rightarrow \mathbb{R}$ is equal to
$\|f\|_1=\int \limits_{x \in \mathcal{D}}|f(x)|dx$ assuming that the
integral is finite. The 
$L_1$ norm can be generalized by taking the integral with respect to the
a nonnegative, finite measure $\mu$, i.e. $\|f\|_{1_\mu}=\int \limits_{x\in
\mathcal{D}}|f(x)|d\mu(x)$.

\section{Kernels on sets}

Let $\mathcal{D}$ be a set. A set system $\mathscr{X}$ is defined as a subset of
the power set $2^{\mathcal{D}}$, i.e. a set of subsets of $\mathcal{D}$.
We assume that $\mathscr{X}$ is a set algebra,
thus it is closed on countable infinite intersections and unions,
namely for finite examples all $A\in \mathscr{X}$ and $B\in
\mathscr{X}$ , $A \cap B\in \mathscr{X}$ and
$A \cup B\in \mathscr{X}$. Furthermore for every
subset $S\in \mathscr{X}$ $\mathcal{D}\setminus
S\in \mathscr{X}$ holds. The set $\mathcal{D}$, consequently the $\emptyset$, is
in $\mathscr{X}$ as well. We might also assume that if $\mathcal{D}$
is finite set then for all $x\in \mathcal{D}$  the singleton set $\{x\}$
is in $\mathscr{X}$ as well. 

Suppose that there is a finite measure $\mu$ defined on $\mathcal{D}$, i.e. 
$\mu(\mathcal{D})<\infty$, and all elements of $\mathscr{X}$ are measurable with
a positive measure except for the $\emptyset$. If $I_{A}$ denotes the indicator
function of a subset $A$ of $\mathcal{D}$ and $A\in\mathscr{X}$ then $\mu(A)=\int \limits_{x\in
\mathcal{D}} I_{A}(x)d\mu(x)$. If $\mathcal{D}$ is a finite set, $\mu$ might be
defined as a counting measure on the singletons. 

We can define a positive definite inner product between sets, and rely
on those inner products to develop kernel based machine learning methods. 
For more background and several other alternative realizations refer
to the book \cite{Shawe-Taylor2004}. 
The examples mentioned here can only demonstrate the broad range of possible
kernels defined on sets. For us the following two cases are important.  
Let $A,B \in \mathscr{X}$ be arbitrary sets, then we can define these kernels: 
\begin{itemize}
\item {\bf Intersection kernel:}
  The basic case assuming a counting measure is given by
  \begin{equation}
    \kappa_{\cap}(A,B)=|A\cap B|,
  \end{equation}
  and the general case can be written in this form 
  \begin{equation}
   \begin{array}{ll}
    \kappa_{\cap}(A,B)_{\mu}& =\mu(A\cap B)
    =\int \limits_{x\in \mathcal{D}} I_{A\cap B}(x)d\mu(x)
    =\int \limits_{x\in \mathcal{D}} I_{A}(x) I_{B}(x) d\mu(x) \\
    & = \braket{I_{A}(x),I_{B}(x)}_{\mu}.
   \end{array}
  \end{equation}
\item {\bf Tanimoto kernel:}
  The counting measure based version is given by
  \begin{equation}
    \kappa_{tanimoto}(A,B)=\dfrac{|A\cap B|}{|A\cup B|},
  \end{equation}  
  and the case relating to a general measure $\mu$ is the following  
  \begin{equation}
    \kappa_{tanimoto}(A,B)_{\mu}=\dfrac{\mu(A\cap B)}{\mu(A\cup B)}
   =\dfrac{\int \limits_{x\in \mathcal{D}} I_{A\cap B}(x)d\mu(x)}{\int \limits_{x\in
   \mathcal{D}} I_{A\cup B}(x)d\mu(x)}. 
  \end{equation}  
\end{itemize}

In the sequel, for sake of simplicity the subscript $_{\mu}$ might be
dropped from the notation of the kernels whenever it is clear which
measure is applied.

\section{Main results}

\label{sec:main_results}

In this Section a summary of the results are presented. We start with
the known extension of the Tanimoto kernel constructed on the
nonnegative functions, i.e. $f:\mathcal{D} \rightarrow \mathbb{R}_{+}$
(c.f. \cite{Ioffe36928}), and subsequently present the generalized
version for the arbitrary real-valued case. The details of the
derivation are unfolded in the following Section~\ref{sec:background}. 

\subsection{Nonnegative functions}

We can represent a nonnegative function $f$ using a set based description. 
Let $A_f$ be a set assigned to $f$, and it is defined by:
\begin{equation}
A_f=\{ (x,y) | x\in \mathcal{D}, y\in \mathbb{R}, 0\le y \le f(x)\}
\subset \mathcal{D} \times \mathbb{R}.
\end{equation}
The set $A_f$ represents the area between the graph of the function and 
the $x$ axes. For a pair of nonnegative functions $f$ and $g$  we can
define the extended intersection and the Tanimoto kernels utilizing the set 
based description.
\begin{proposition}
For any nonnegative functions $f$ and $g$ defined on the domain
$\mathcal{D}$, and for the finite measure $\mu$ we have the
intersection kernel  
\begin{equation}
\label{eq:intersection_kernel}
\kappa_{\cap}(f,g) = \mu(A_f \cap A_g)=\int \limits_{x\in \mathcal{D}}
\min(f(x),g(x))d\mu(x),  
\end{equation} 
and the Tanimoto kernel
\begin{equation}
\label{eq:minmax_kernel}
\kappa(f,g)_{tanimoto} = \dfrac{\mu(A_f \cap A_g)}{\mu(A_f \cup
A_g)} 
=\dfrac{\int \limits_{x\in \mathcal{D}} \min(f(x),g(x))d\mu(x)}{\int \limits_{x\in \mathcal{D}} \max(f(x),g(x))d\mu(x)}.
\end{equation} 
\end{proposition}
This kernel is generally referred as {\it MinMax kernel}, see for example
\cite{Ioffe36928}.   

\subsection{Arbitrary real-valued functions}
 
Now we move on to the main result of this paper. For the general case
we need to extend the definition of the set  $A_f$ 
assigned to an arbitrary real valued function $f:\mathcal{D}
\rightarrow \mathbb{R}$. To this end we
split the domain $\mathcal{D}$ of $f$ based on the sign of the
function values. Let $f=f^{+}+f^{-}$ where 
\begin{equation}
\begin{array}{ll}
f^{+}(x) =
\left \{
\begin{array}{ll}
f(x) & f(x)\ge 0, \\
0    & f(x) < 0,
\end{array}
\right .
& 
f^{-}(x) =
\left \{
\begin{array}{ll}
-f(x) & f(x) < 0, \\
0    & f(x) \ge 0.
\end{array}
\right .
\end{array}
\end{equation}
Clearly, $f^{+}$ and $(-f)^{-}$ are nonnegative
functions.  
Let 
\begin{equation}
  \label{def:Ap_f}
  \begin{array}{@{}ll@{}}
    A^{+}_{f}&= \{ (x,y) | x\in \mathcal{D}, y\in \mathbb{R}, 0\le y
               \le f^{+}(x)\},
  \end{array}
\end{equation}
and
\begin{equation}
  \label{def:An_f}
  \begin{array}{@{}ll@{}}
    A^{-}_{f}&= \{ (x,y) | x\in \mathcal{D}, y\in \mathbb{R},\ f^{-}(x)
               \le y < 0 \},
  \end{array}
\end{equation}
be the corresponding sets. Note that $A^{+}_f \cap
A^{-}_f=\emptyset$, and we can define $A_f$ as
\begin{equation}
  \label{def:A_f}
    A_{f}=A^{+}_f \cup A^{-}_f. 
\end{equation} 

We can construct the corresponding kernels for the general case.
\begin{proposition}
  Let $f$ and $g$ be any real valued functions on domain $\mathcal{D}$,
  and for the finite measure $\mu$ we have the intersection
  kernel 
\begin{equation}
\label{eq:gen_intersection_kernel}
\renewcommand{\arraystretch}{1.5}
\begin{array}{@{}l@{\:}l@{}}
\kappa_{\cap}(f,g) & = \mu(A_f \cap A_g)
=\frac{1}{2}(\|f\|_{1_\mu}+\|g\|_{1_\mu}-\|f-g\|_{1_\mu})
  \\ 
& = \int \limits_{\min(f(x),g(x))\ge 0} \min(f(x),g(x)) d\mu(x)  
- \int \limits_{\max(f(x),g(x))<0} \max(f(x),g(x))d\mu(x). 
\end{array}
\renewcommand{\arraystretch}{1}
\end{equation} 
and the Tanimoto kernel
\begin{equation}
\label{eq:general_tanimoto_kernel}
\renewcommand{\arraystretch}{3.5}
\begin{array}{@{}l@{\:}l@{}}
\kappa_{tanimoto}(f,g)
& =\dfrac{\mu(A_f \cap A_g)}{\mu(A_f \cup A_g)} 
=\dfrac{\frac{1}{2}(\|f\|_{1_\mu}+\|g\|_{1_\mu}-\|f-g\|_{1_\mu})}
{\frac{1}{2}(\|f\|_{1_\mu}+\|g\|_{1_\mu}+\|f-g\|_{1_\mu})}  \\
& = \dfrac{
\int \limits_{\min(f(x),g(x))\ge 0} 
\min(f(x),g(x)) d\mu(x)  
- \int \limits_{\max(f(x),g(x))<0} 
\max(f(x),g(x))d\mu(x)
}
{\int \limits_{\max(f(x),g(x))\ge 0} 
\max(f(x),g(x)) d\mu(x)  
 -\int \limits_{\min(f(x),g(x))<0} 
\min(f(x),g(x)) d\mu(x)}.  
\end{array}
\renewcommand{\arraystretch}{1}
\end{equation}
\end{proposition}

\section{Background}

\label{sec:background}

In Section \ref{sec:main_results} to a function $f$ we assign a set
$A_f$ defined on the area between the graph of the function and the
$X$-axis. Here we show a slightly different approach to define the same $A_f$
which allows to reformulate the expression  
(\ref{eq:general_tanimoto_kernel}) of the general Tanimoto kernel.
In the description of the alternative form we exploit the concept of
{\it epigraph} of a function. The epigraph and plays a
central role 
in convex analysis, \cite{Rockafellar1970},  and in the variational
analysis of general functions, and multi-valued mappings,
\cite{Rockafellar1997}.  

\begin{figure}
  \begin{center}
    \includegraphics[width=4in,height=2.8in]{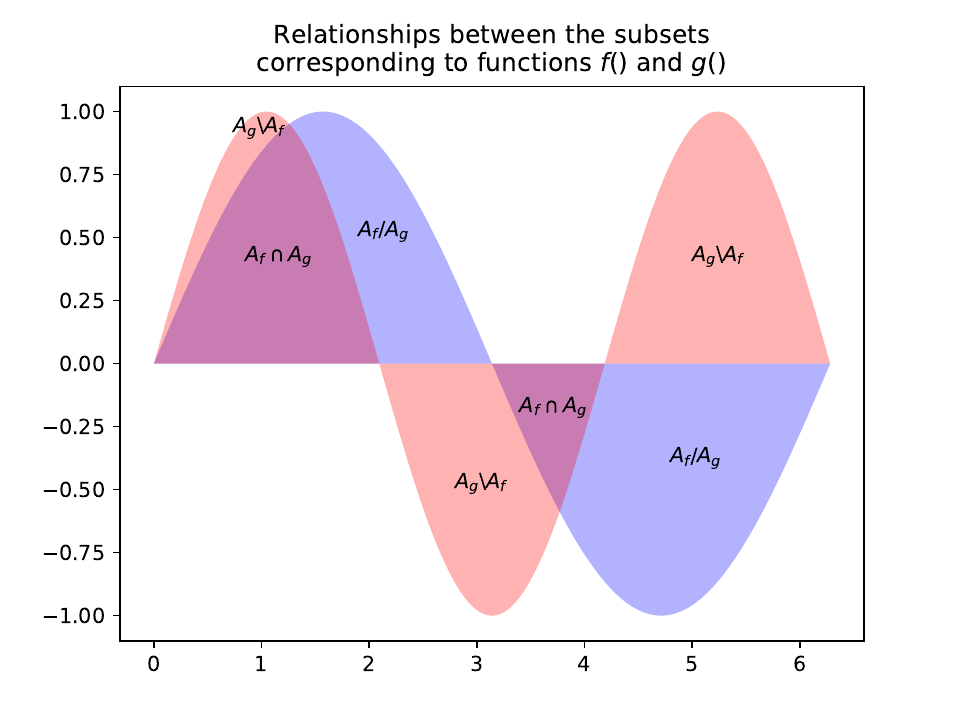}
  \end{center}
  \caption{The interrelations of the areas, $A_f$ and $A_g$, defined
  via the functions $f$ and $g$. }
\label{fig:gen_geometry}
\end{figure}

\begin{figure}
  \begin{center}
    \includegraphics[width=4in,height=2.3in]{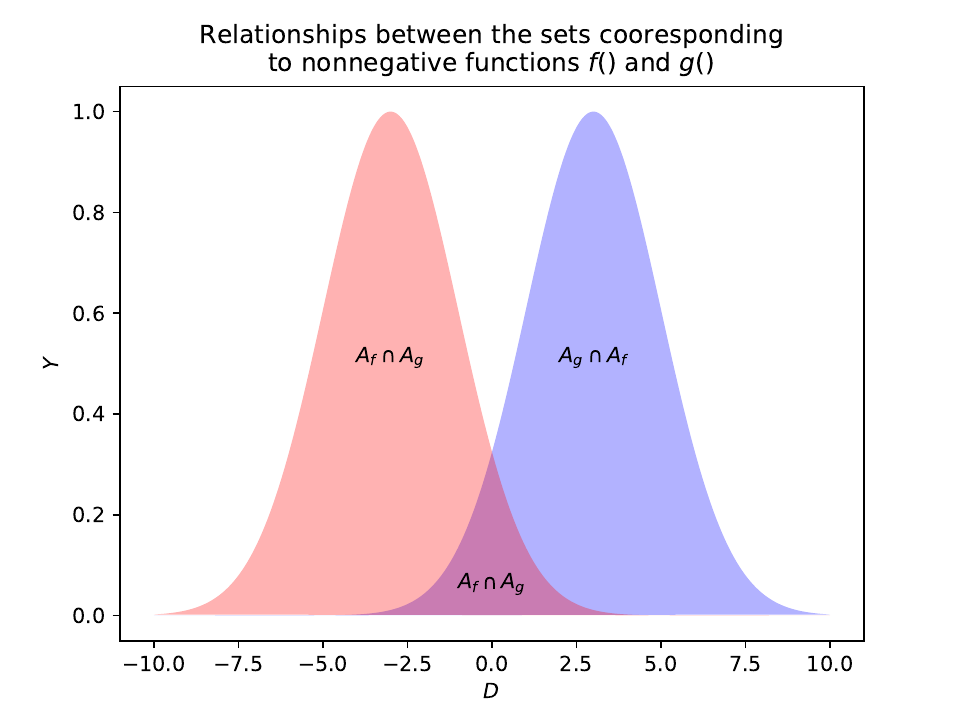}
  \end{center}
  \caption{The interrelations of the areas, $A_f$ and $A_g$, defined
  via the non-negative functions $f$ and $g$. }
\label{fi_nonnegative_areas}
\end{figure}

We define a function in the same way, $f: \mathcal{D} \rightarrow
\mathbb{R}$, as we did earlier. 
The epigraph of $f$ is a set and given by 
\begin{equation}
\mathscr{E}_f= \{ (x,y) | y \ge f(x), x\in \mathcal{D} \}.  
\end{equation} 
The graph of a function is the set 
\begin{equation}
\mathscr{G}_f= \{ (x,y) | y=f(x), x\in \mathcal{D} \}.  
\end{equation}
The hypograph of $f$ of is equal to 
\begin{equation}
\mathscr{H}_f= \{ (x,y) | y \le f(x), x\in \mathcal{D}
\}=\mathscr{G}_{f} \cup \mathscr{E}_{f},  
\end{equation} 
which is a complement of the epigraph which also comprises the graph of the
function.   
For any pair of bounded functions $f$ and $g$ based on the definition of the
epigraph and hypograph we can derive these equations
\begin{equation}
\label{eq:epi_hypo_min_max}
\begin{array}{llll}
\mathscr{H}_f \cap \mathscr{H}_g & = \mathscr{H}_{\min(f,g)}, &
\mathscr{E}_f \cap \mathscr{E}_g & = \mathscr{E}_{\max(f,g)}, \\
\mathscr{H}_f \cup \mathscr{H}_g & = \mathscr{H}_{\max(f,g)}, &
\mathscr{E}_f \cap \mathscr{E}_g & = \mathscr{E}_{\min(f,g)}. \\
\end{array}
\end{equation}

Assuming that the measure $\mu$ is finite, and $f$ is bounded from
above and below, we also have the following
statements. 
\begin{itemize}
\item 
For any $f$ we have 
$\mu_{\mathcal{D}\times \mathbb{R}}(\mathcal{G}_f)=0$. 
\item As a consequence $
\mu_{\mathcal{D}\times \mathbb{R}}(\mathcal{E}_f \cap
\mathcal{H}_f)=0$ 
 holds as well.
\item Another consequence in case of any sets $A,B \subset \mathcal{D} \times
  \mathbb{R}$
\begin{equation}
\label{eq:subset_epi_hypo}
\mu_{\mathcal{D}\times \mathbb{R}}((A \cap B \cap \mathcal{E}_f \cap \mathcal{H}_f))=0
\end{equation} 
 is also true, since $A \cap B \cap \mathcal{E}_f \cap
 \mathcal{H}_f\subseteq \mathcal{E}_f \cap \mathcal{H}_f$. 
\end{itemize}

With these concepts at hand we can redefine $A_f$ as $A_f=(\mathscr{H}_f
\cap \mathscr{E}_0) \cup (\mathscr{E}_f \cap \mathscr{H}_0)$, where
$\mathcal{E}_0$ is the epigraph of the constant, $0$ valued function. 
Note that the redefined $A_f$ covers the same set as earlier. For sake of
simplicity, in the following expression we might drop the intersection
operator $\cap$, thus for any sets $A$ and $B$, $AB$  means $A\cap B$. 
If two functions $f$ and $g$ are given on the same domain, the intersection
of the corresponding sets is 
\begin{equation}
\begin{array}{ll}
A_f \cap A_g & = 
(\mathscr{H}_f\mathscr{E}_0 \cup \mathscr{E}_f\mathscr{H}_0)
(\mathscr{H}_g\mathscr{E}_0 \cup \mathscr{E}_g\mathscr{H}_0)
\\
& = \mathscr{H}_f\mathscr{H}_g\mathscr{E}_0 \cup
  \mathscr{E}_f\mathscr{H}_g\mathscr{H}_0\mathscr{E}_0 \cup
  \mathscr{H}_f\mathscr{E}_g\mathscr{E}_0\mathscr{H}_0 \cup 
\mathscr{E}_f\mathscr{E}_g\mathscr{H}_0.  
\end{array}
\end{equation} 
Exploiting the equality (\ref{eq:subset_epi_hypo}), we can see that the 
terms containing the intersection $\mathscr{E}_0 \cap
\mathscr{H}_0$ have to have $0$ measure with respect to $\mu_{\mathcal{D}
\times \mathbb{R}}$. Hence we can state that 
\begin{equation}
\renewcommand{\arraystretch}{1.5}
\begin{array}{@{}l@{\:}l@{}}
\mu_{\mathcal{D} \times \mathbb{R}}(A_f \cap A_g) 
& = \mu_{\mathcal{D} \times \mathbb{R}}(\mathscr{H}_f\mathscr{H}_g\mathscr{E}_0)
+ \mu_{\mathcal{D} \times
  \mathbb{R}}(\mathscr{E}_f\mathscr{E}_g\mathscr{H}_0) \\
& =\int \limits_{\substack{x \in \mathcal{D} \\ \min(f(x),g(x))\ge 0 }}
  \min(f(x),g(x)) d\mu(x)  
- \int \limits_{\substack{x \in \mathcal{D} \\ \max(f(x),g(x))< 0 }}
  \max(f(x),g(x)) d\mu(x).
\end{array}
\renewcommand{\arraystretch}{1}
\end{equation} 
In case of the union of $A_f$ and $A_g$, we can similarly write  
\begin{equation}
\begin{array}{ll}
A_f \cup A_g & = 
\mathscr{H}_f\mathscr{E}_0 \cup \mathscr{E}_f\mathscr{H}_0 \cup
\mathscr{H}_g\mathscr{E}_0 \cup \mathscr{E}_g\mathscr{H}_0 \\
& = (\mathscr{H}_f \cup \mathscr{H}_g)\mathscr{E}_0 \cup 
  (\mathscr{E}_f \cup \mathscr{E}_g)\mathscr{H}_0.
\end{array}
\end{equation} 
By applying (\ref{eq:subset_epi_hypo}) again, the intersection of
$(\mathscr{H}_f \cup \mathscr{H}_g)\mathscr{E}_0$ and $(\mathscr{E}_f
\cup \mathscr{E}_g)\mathscr{H}_0$ has zero measure with respect to
$\mu_{\mathcal{D} \times \mathbb{R}}$. Therefore we have
\begin{equation}
\renewcommand{\arraystretch}{1.5}
\begin{array}{ll}
\mu_{\mathcal{D} \times \mathbb{R}}(A_f \cup A_g) 
& = \mu_{\mathcal{D} \times \mathbb{R}}((\mathscr{H}_f \cup \mathscr{H}_g)\mathscr{E}_0)
+ \mu_{\mathcal{D} \times \mathbb{R}}((\mathscr{E}_f\cup
  \mathscr{E}_g)\mathscr{H}_0) \\  
& =\int \limits_{\substack{x \in \mathcal{D} \\ \max(f(x),g(x))\ge 0 }}
  \max(f(x),g(x)) d\mu(x)  
- \int \limits_{\substack{x \in \mathcal{D} \\ \min(f(x),g(x))< 0 }}
  \min(f(x),g(x)) d\mu(x).
\end{array}
\renewcommand{\arraystretch}{1}
\end{equation} 

The underlying geometry of the general case is highlighted on Figure
\ref{fig:gen_geometry}. Figure \ref{fi_nonnegative_areas} demonstrates
the simpler geometric setting exploited in the case of nonnegative functions.   

Now we can express the general
Tanimoto kernel by exploiting the equations of
(\ref{eq:epi_hypo_min_max}) connecting the pointwise minimum and
maximum of functions to the intersection of epigraphs and the hypographs.  
\begin{equation}
\begin{array}{ll}
\label{eq:general_minmax_tanimoto_kernel}
\kappa_{tanimoto}(f,g)
 =\dfrac{\mu(A_f \cap A_g)}{\mu(A_f \cup A_g)} \\
\quad =\dfrac{
\int \limits_{\min(f(x),g(x))\ge 0} 
\min(f(x),g(x)) d\mu(x)  
- \int \limits_{\max(f(x),g(x))<0} 
\max(f(x),g(x))d\mu_{\mathcal{D}}(x)
}
{\int \limits_{\max(f(x),g(x))\ge 0} 
\max(f(x),g(x)) d\mu_{\mathcal{D}}(x)  
 -\int \limits_{\min(f(x),g(x))<0} 
\min(f(x),g(x)) d\mu_{\mathcal{D}}(x)}  
\end{array}
\end{equation}
Both in the numerator and the denominator within the second terms we
can use strict inequalities defining the integral domain, since the
integrals of bounded functions on the boundaries of the integral
domains are equal to $0$. 

In the remaining part of this section we show that the generalized
Tanimoto kernel can also be described by the help of $L_1$ norm. 
We have the following statements. 
\begin{proposition}
\label{prop_tanimoto_general}
For any functions $f,g:\mathcal{D} \rightarrow \mathbb{R}$ we have for
the intersection
\begin{equation}
\label{eq:gen_fcapg_L1}
\renewcommand{\arraystretch}{1.5}
\begin{array}{ll}
\frac{1}{2}(||f||_{1_{\mu}}+||g||_{1_{\mu}}-||f-g||_{1_{\mu}}) \\
\qquad = 
\int \limits_{\min(f(x),g(x))\ge 0} 
\min(f(x),g(x)) d\mu_{\mathcal{D}}(x)  
- \int \limits_{\max(f(x),g(x))<0} 
\max(f(x),g(x))d\mu_{\mathcal{D}}(x) \\
\qquad = \mu(A_f \cap A_g),
\end{array}
\renewcommand{\arraystretch}{1}
\end{equation}
and for the union
\begin{equation}
\label{eq:gen_fcupg_L1}
\renewcommand{\arraystretch}{1.5}
\begin{array}{ll}
\frac{1}{2}(||f||_{1_{\mu}}+||g||_{1_{\mu}}+||f-g||_{1_{\mu}}) \\
\qquad = \int \limits_{\max(f(x),g(x))\ge 0} 
\max(f(x),g(x)) d\mu_{\mathcal{D}}(x)  
 -\int \limits_{\min(f(x),g(x))<0} 
\min(f(x),g(x)) d\mu_{\mathcal{D}}(x) \\  
\qquad = \mu(A_f \cup A_g).
\end{array}
\renewcommand{\arraystretch}{1}
\end{equation}
\end{proposition}
\begin{proof}

By inspection for any $x\in \mathcal{D}$ we can claim the following identities
\begin{equation}
\begin{array}{lll}
f(x)+g(x)& =\max(f(x),g(x))+\min(f(x),g(x)), \\
f(x)-g(x)& =\max(f(x),g(x))-\min(f(x),g(x)), & \text{if}\ f(x)\ge g(x),  \\
|f(x)-g(x)|& =\max(f(x),g(x))-\min(f(x),g(x)). \\
\end{array}
\end{equation}

The sub-terms of the $L_1$ norm based forms can be expressed by the
help of the minimum and maximum functions. 
\begin{equation}
\renewcommand{\arraystretch}{1.5}
\begin{array}{ll}
||f-g||_{1_{\mu}}& =\int \limits_{x\in \mathcal{D}} |f(x)-g(x)| d\mu(x)
  \\
& = \int \limits_{x\in \mathcal{D}}
  \left(\max(f(x),g(x))-\min(f(x),g(x))\right) d\mu(x)  \\
& = + \int \limits_{\substack{ x\in \mathcal{D} \\ \max(f(x),g(x))\ge 0}} 
    \max(f(x),g(x) d\mu(x) 
   + \int \limits_{\substack{ x\in \mathcal{D} \\ \max(f(x),g(x))< 0}} 
    \max(f(x),g(x) d\mu(x) \\
& \quad - \int \limits_{\substack{ x\in \mathcal{D} \\ \min(f(x),g(x))\ge 0}} 
    \min(f(x),g(x) d\mu(x) 
   - \int \limits_{\substack{ x\in \mathcal{D} \\ \min(f(x),g(x))< 0}} 
    \min(f(x),g(x) d\mu(x) \\
& = + \int \limits_{\substack{ x\in \mathcal{D} \\ \max(f(x),g(x))\ge 0}} 
    \max(f(x),g(x) d\mu(x)
- \int \limits_{\substack{ x\in \mathcal{D} \\ \min(f(x),g(x))< 0}} 
    \min(f(x),g(x) d\mu(x) \\
& \quad - \int \limits_{\substack{ x\in \mathcal{D} \\ \min(f(x),g(x))\ge 0}} 
    \min(f(x),g(x) d\mu(x)
   + \int \limits_{\substack{ x\in \mathcal{D} \\ \max(f(x),g(x))< 0}} 
    \max(f(x),g(x) d\mu(x) \\
& = \mu(A_f \cup A_g) - \mu(A_f \cap A_g)
\end{array}
\renewcommand{\arraystretch}{1}
\end{equation}
The sum of the $L_1$ norms can be expressed by
\begin{equation}
\label{eq:sumofL1norms}
\renewcommand{\arraystretch}{1.5}
\begin{array}{@{}l@{\:}l@{}}
||f||_{1_{\mu}}+||g||_{1_{\mu}}
& =\int \limits_{x\in \mathcal{D}} |f(x)|+|g(x)| d \mu(x) \\
& = + \int \limits_{\substack{x \in \mathcal{D} \\f(x)\ge 0,g(x)\ge 0, }} f(x)+g(x) d\mu(x)
- \int \limits_{\substack{x \in \mathcal{D} \\f(x)<0,g(x)< 0, }} 
  f(x)+g(x) d\mu(x) \\
& \quad +\int \limits_{\substack{x \in \mathcal{D} \\f(x)\ge 0,g(x)< 0, }} f(x)-g(x) d\mu(x)
+ \int \limits_{\substack{x \in \mathcal{D} \\f(x)<0,g(x)\ge 0, }} g(x)-f(x) d\mu(x) \\ 
& = + \int \limits_{\substack{x \in \mathcal{D} \\f(x)\ge 0,g(x)\ge 0,
  }} \max(f(x),g(x))+\min(f(x),g(x)) d\mu(x) \\
& \quad - \int \limits_{\substack{x \in \mathcal{D} \\f(x)<0,g(x)< 0, }} 
  \max(f(x),g(x))+\min(f(x),g(x)) d\mu(x) \\
& \quad +\int \limits_{\substack{x \in \mathcal{D} \\f(x)\ge 0,g(x)<
  0, }} \max(f(x),g(x))-\min(f(x),g(x)) d\mu(x) \\
& \quad + \int \limits_{\substack{x \in \mathcal{D} \\f(x)<0,g(x)\ge
  0, }} \max(f(x),g(x))-\min(f(x),g(x)) d\mu(x). \\
\end{array}
\renewcommand{\arraystretch}{1}
\end{equation}
By applying these identities
\begin{equation}
\renewcommand{\arraystretch}{1.5}
\begin{array}{@{}l@{\:}l@{}}
\{x | \max(f(x),g(x))\ge 0\} =\{x| f(x)\ge 0,g(x)\ge 0\} 
\cup \{x| f(x)\ge 0,g(x)< 0\} \cup \{x| f(x)< 0,g(x)\ge 0\}, \\   
\{x | \max(f(x),g(x))< 0\} =\{x| f(x)< 0,g(x)< 0\}, \\    
\{x | \min(f(x),g(x))\ge 0\} =\{x| f(x)\ge 0,g(x)\ge 0\}, \\ 
\{x | \min(f(x),g(x))< 0\} =\{x| f(x)< 0,g(x)< 0\} 
\cup \{x| f(x)\ge 0,g(x)< 0\} \cup \{x| f(x)< 0,g(x)\ge 0\}, \\   
\end{array}
\renewcommand{\arraystretch}{1}
\end{equation}
we can rearrange the terms in (\ref{eq:sumofL1norms}), which gives us
\begin{equation}
\renewcommand{\arraystretch}{1.5}
\begin{array}{@{}l@{\:}l@{}}
||f||_{1_{\mu}}+||g||_{1_{\mu}} 
& = + \int \limits_{\substack{ x\in \mathcal{D} \\ \max(f(x),g(x))\ge 0}} 
    \max(f(x),g(x) d\mu(x)
- \int \limits_{\substack{ x\in \mathcal{D} \\ \min(f(x),g(x))< 0}} 
    \min(f(x),g(x) d\mu(x) \\
& \quad - \int \limits_{\substack{ x\in \mathcal{D} \\ \min(f(x),g(x))\ge 0}} 
    \min(f(x),g(x) d\mu(x)
   + \int \limits_{\substack{ x\in \mathcal{D} \\ \max(f(x),g(x))< 0}} 
    \max(f(x),g(x) d\mu(x) \\
& = \mu(A_f \cup A_g) - \mu(A_f \cap A_g).
\end{array}
\renewcommand{\arraystretch}{1}
\end{equation}
After combining the sub-terms together, we can arrive at the final
formulas. For the intersection we have  
\begin{equation}
\begin{array}{ll}
\frac{1}{2}(||f||_{1_{\mu}}+||g||_{1_{\mu}}-||f-g||_{1_{\mu}})
& =\frac{1}{2}(\mu(A_f \cup A_g) - \mu(A_f \cap A_g)-\mu(A_f \cup A_g)
  + \mu(A_f \cap A_g)) \\
& =\mu(A_f \cap A_g)),  
\end{array}
\end{equation}
and the union takes this form 
\begin{equation}
\begin{array}{ll}
\frac{1}{2}(||f||_{1_{\mu}}+||g||_{1_{\mu}}+||f-g||_{1_{\mu}})
&=\frac{1}{2}(\mu(A_f
\cup A_g) - \mu(A_f \cap A_g)+\mu(A_f \cup A_g) -\mu(A_f \cap
A_g)) \\
&=\mu(A_f \cup A_g)).   
\end{array}
\end{equation}

\end{proof}

From Proposition \ref{prop_tanimoto_general} we can straightforwardly
derive the MinMax kernel for the nonnegative functions, see in
Proposition \ref{eq:minmax_kernel}. If the functions are nonnengative
then the second terms in both the numerator and the denominator of
(\ref{eq:general_minmax_tanimoto_kernel}) can be dropped. 
\begin{corollary}
For any nonnegative functions $f,g:\mathcal{D} \rightarrow \mathbb{R}$
the generalized Tanimoto kernel takes the following form
\begin{equation}
\begin{array}{ll}
\label{eq:minmax_tanimoto_kernel}
\kappa_{\cap}(f,g)
 =\mu(A_f \cap A_g) 
\quad =\dfrac{
\int \limits_{x \in \mathcal{D}} 
\min(f(x),g(x)) d\mu(x)  
}
{\int \limits_{x \in \mathcal{D}} 
\max(f(x),g(x)) d\mu_{\mathcal{D}}(x)}  
\end{array}
\end{equation}

\end{corollary}

\section{Explicit feature from Tanimoto kernel}

 
The inner product implied by the Tanimoto kernel can be defined on
any pairs of subsets $A,B\in \mathscr{X}$ by
\begin{equation}
  \braket{A,B}_{\mu}=\dfrac{\mu(A\cap B)} {\mu(A\cup B)}.
\end{equation}
If $A$ or $B$ is empty than the $\braket{A,B}=0$. In the sequel we
assume that $A\cup B\ne \emptyset$. We can also write the inner
product in another form 
\begin{equation}
  \braket{A,B}_{\mu}=\dfrac{\mu(A\cap B)}{\mu(\mathcal{D}\setminus (\overline{A}\cap \overline{B}))}=\dfrac{\mu(A\cap B)}{\mu(\mathcal{D})-\mu(\overline{A}\cap \overline{B})}=\dfrac{\frac{\mu(A\cap B)}{\mu(\mathcal{D})}}{1-\frac{\mu(\overline{A}\cap \overline{B})}{\mu(\mathcal{D})}},
\end{equation}
where for a set $A \in \mathscr{X}$ we have $\overline{A}=\mathcal{D}\setminus
A$. 
Now suppose that $0< \mu(\overline{A}\cap \overline{B}) <
\mu(\mathcal{D})$, i.e. $\emptyset \ne A\cup B \subset \mathcal{D}$,
then  we have $0<\frac{\mu(\overline{A}\cap
\overline{B})}{\mu(\mathcal{D})}<1$. As a consequence we can write the
inner product as a geometric series with initial value
$\frac{\mu(A\cap B)}{\mu(\mathcal{D})}$ and with factor
$\frac{\mu(\overline{A}\cap \overline{B})}{\mu(\mathcal{D})}$, thus
it yields
\begin{equation}
  \label{eq_tanimoto_unfold}
  \begin{array}{ll}
    \braket{A,B}_{\mu}&=\dfrac{\mu(A\cap B)} {\mu(A\cup B)}
  =\dfrac{\mu(A\cap B)}{\mu(\mathcal{D})-\mu(\overline{A}\cap \overline{B})}
  =\dfrac{\frac{\mu(A\cap B)}{\mu(\mathcal{D})}}
   {1-\frac{\mu(\overline{A}\cap \overline{B})}{\mu(\mathcal{D})}} \\
  &=\dfrac{\mu(A\cap B)}{\mu(\mathcal{D})}
    \left(1+\dfrac{\mu(\overline{A}\cap \overline{B})}{\mu(\mathcal{D})}
    + \dots + \left(\dfrac{\mu(\overline{A}\cap \overline{B})}
    {\mu(\mathcal{D})} \right)^k + \dots  \right) \\
  &=\dfrac{\mu(A\cap B)}{\mu(\mathcal{D})}
    \sum \limits_{k=1}^{\infty} \left(\dfrac{\mu(\overline{A}\cap \overline{B})}
    {\mu(\mathcal{D})}\right)^{k-1}. 
  \end{array}
\end{equation}

From the Expression (\ref{eq_tanimoto_unfold}) of the inner product we
can derive an 
explicit feature representation $\phi()$ which maps $\mathscr{X}$ into
a Hilbert space. To this end we need two well known
properties of the inner product.

Let $H_1$ and $H_2$ two Hilbert spaces, then for any two pairs $a,b\in
H_1$ and $c,d\in H_2$, we have these identities.
\begin{itemize}
  \item {\bf Sum of inner products:}
\begin{equation}
  \braket{a,b}+\braket{c,d}=\braket{a \oplus c,b \oplus d},
\end{equation}
where $a \oplus c,b \oplus d\in H_1\oplus H_2$, $\oplus$ denotes the
direct sum of the Hilbert spaces. In the finite dimensional case the
direct sum means the concatenation of the corresponding vectors.  
\item {\bf Product of inner products:}
\begin{equation}
  \braket{a,b}\braket{c,d}=\braket{a \otimes c,b \otimes d},
\end{equation}
where $a \otimes c,b \otimes d\in H_1\otimes H_2$, $\otimes$ denotes the
tensor product of the Hilbert spaces. In the finite dimensional case the
tensor product means the outer product of the corresponding vectors.  
\end{itemize}

Additionally, since $\mu$ is a counting  measure we have
\begin{equation}
  \begin{array}{l}
    \mu(A\cap B)=\braket{\mathbb{I}(R),\mathbb{I}(S)}_{\mu}\quad
    \text{and}\quad 
    \mu(\overline{A}\cap \overline{S})
    =\braket{\mathbb{I}(\overline{A}),\mathbb{I}(\overline{B})}_{\mu}, 
  \end{array}
\end{equation}
where $\mathbb{I}()$ is the indicator function of the subsets of
$\mathcal{D}$, thus it can be represented via a binary vector.
Now based on the two above mentioned  properties of the inner product
Expression (\ref{eq_tanimoto_unfold}) can be decomposed into
an explicit feature such that 
\begin{equation}
  \begin{array}{@{}l@{}l@{}}
    \phi(\mathbb{I}(A))& =\frac{1}{\sqrt{\mu(\mathcal{D})}}
      \left( \mathbb{I}(A)
      \oplus \left[\mathbb{I}(A)
      \otimes \frac{\mathbb{I}(\overline{A})}{\sqrt{\mu(\mathcal{D})}}\right]
      \oplus \left[\mathbb{I}(A)\otimes (\otimes^2
      \frac{\mathbb{I}(\overline{A}))}{\sqrt{\mu(\mathcal{D})}}\right] 
      \oplus \dots  \right) \\
    &= \frac{1}{\sqrt{\mu(\mathcal{D})}} \left( \bigoplus_{k=1}^{\infty} \mathbb{I}(A)
      \bigotimes \otimes_{l=1}^{k-1}\frac{\mathbb{I}(\overline{A})}{\sqrt{\mu(\mathcal{D})}} \right),  
  \end{array}
\end{equation}
for all $A\subseteq \mathcal{D}$.

\section{Tanimoto kernel from general other type of kernels}

Let $\mathcal{S}=\{x_i|i=1,\dots,m\}$ be a sample of a set
$\mathcal{X}$, and we are given a feature representation $\phi:
\mathcal{X} \rightarrow \mathcal{H}$ in a Hilbert space $\mathcal{H}$. The
kernel function in $\mathcal{H}$ is denoted by $\kappa_{H}$.
 Let $\mathcal{B}$ be a set of basis elements of
$\mathcal{X}$. $\mathcal{B}$ which might be created by randomly
subsampling $\mathcal{S}$. We can
construct a basis relative feature vector for every $x\in \mathcal{X}$, by
\begin{equation}
\phi[B](x)=(\kappa_{H}(x,b_k)| b_k \in \mathcal{B}).   
\end{equation}
For any to elements $u$ and $v$ of $\mathcal{X}$ the basis relative
minimum and maximum function can be naturally defined as
\begin{equation}
\begin{array}{ll}
\min[\mathcal{B},\kappa_{H}](u,v)& 
= (\min(\kappa_{H}(u,b_k),\kappa_{H}(v,b_k))|b_k\in \mathcal{B}), \\     
\max[\mathcal{B},\kappa_{H}](u,v)& 
= (\max(\kappa_{H}(u,b_k),\kappa_{H}(v,b_k))|b_k\in \mathcal{B}).      
\end{array}
\end{equation}
By the help of these definitions we can write up the general Tanimoto kernel on
the top of kernel function $\kappa_{H}$.
\begin{equation}
\begin{array}{ll}
\label{eq:general_minmax_basis_tanimoto_kernel}
\kappa_{tanimoto}[\mathcal{B},\kappa_{H}](u,v) \\
\quad =\dfrac{
\sum \limits_{\min[\mathcal{B},\kappa_{H}](u_k,v_k)_k\ge 0} 
\min[\mathcal{B},\kappa_{H}](u_k,v_k)_k   
 - \sum \limits_{\max[\mathcal{B},\kappa_{H}](u_k,v_k)_k<0} 
\max[\mathcal{B},\kappa_{H}](u_k,v_k)_k   
}
{
\sum \limits_{\max[\mathcal{B},\kappa_{H}](u_k,v_k)_k\ge 0} 
\max[\mathcal{B},\kappa_{H}](u_k,v_k)_k   
- \sum \limits_{\min[\mathcal{B},\kappa_{H}](u_k,v_k)_k<0} 
\min[\mathcal{B},\kappa_{H}](u_k,v_k)_k.   
}  
\end{array}
\end{equation}

\section{Representation via quotient of piecewise linear functions}  
 
In this section we construct an additional representation of general
Tanimoto kernel for vectors with real components. This representation
is built on a quotient of piecewise linear functions. 
 
We are given a pair of arbitrary vectors $\mbf{a},\mbf{x} \in \mathbb{R}^{n}$. Let 
\begin{equation}
F(\mbf{a},\mbf{x})=\frac{1}{2}(|\mbf{a}|_{1}+|\mbf{x}|_{1}-|\mbf{a}-\mbf{x}|_{1}),
\end{equation}
and 
\begin{equation}
G(\mbf{a},\mbf{x})=\frac{1}{2}(|\mbf{a}|_{1}+|\mbf{x}|_{1}+|\mbf{a}-\mbf{x}|_{1}). 
\end{equation} 
Then we can write the general Tanimoto kernel in this form 
\begin{equation}
\kappa_{tanimoto}(\mbf{a},\mbf{x})
=\dfrac{F(\mbf{a},\mbf{x})}{G(\mbf{a},\mbf{x})}.    
\end{equation}
$F$ and $G$  can also be written as sums, 
\begin{equation}
F(\mbf{a},\mbf{x})=\frac{1}{2}\sum \limits_{j=1}^{n} (F(\mbf{a},\mbf{x}))_j
=\frac{1}{2}\sum \limits_{j=1}^{n}|a_j|+|x_j|-|a_j-x_j|,
\end{equation} 
and similarly 
\begin{equation}
G(\mbf{a},\mbf{x})=\frac{1}{2}\sum \limits_{j=1}^{n} (G(\mbf{a},\mbf{x}))_j
=\frac{1}{2}\sum \limits_{j=1}^{n}|a_j|+|x_j|+|a_j-x_j|. 
\end{equation} 


Let $\mbf{a}$ be fixed and $\kappa_{tanimoto}(\mbf{a},\mbf{x})$ is
expressed as function of $\mbf{x}$. To construct that function,   
we partition the index set $\{1,\dots n\}$ by
the relations comparing the values of $a_j$ and the varying $x_j$ for
any $j=1,\dots,n$. The subsets forming the partition are defined by  
\begin{equation}
\begin{array}{lll}
          & I^{+}_{x_j < 0} & =\{ j | x_j < 0 \}, \\
a_j \ge 0 & I^{+}_{0 \le x_j < a_j}&=\{ j | 0 \le x_j < a_j \}, \\
          & I^{+}_{a_j \le x_j}& =\{ j | a_j \le x_j\}, \\ \hline 
          & I^{-}_{x_j < a_j}&=\{ j | x_j < a_j \}, \\
a_j < 0 & I^{-}_{a_j \le x_j < 0}&=\{ j | a_j \le x_j < 0 \}, \\
          & I^{-}_{0 \le x_j }& =\{ j | 0 \le x_j \}. \\
\end{array}
\end{equation}  
If $\mbf{a},\mbf{x} \in \mathbb{R}^{n}_{+}$, then only $I^{+}_{0 \le x_j < a_j}$ and
$I^{+}_{a_j \le x_j}$ can be nonempty.   

Based on the definition of $F(\mbf{a},\mbf{x})$ and $G(\mbf{a},\mbf{x})$
we can compute the values of $(F(\mbf{a},\mbf{x}))_j$ and
$(G(\mbf{a},\mbf{x}))_j$ for every $j= 1,\dots n$.
\begin{equation}
\begin{array}{@{}l@{\:}|ccc|ccc|@{}}
& \multicolumn{3}{c|}{a_j\ge 0} & \multicolumn{3}{c}{a_j< 0} \\ \hline
j \in & I^{+}_{x_j < 0} & I^{+}_{0 \le x_j < a_j} & I^{+}_{a_j \le x_j}    
& I^{-}_{x_j < a_j} & I^{-}_{a_j \le x_j < 0} & I^{-}_{0 \le x_j} \\ \hline 
(F(\mbf{a},\mbf{x}))_j= & 0 & x_j & a_j & -a_j & -x_j & 0 \\ \hline    
(G(\mbf{a},\mbf{x}))_j= & a_j-x_j & a_j & x_j & -x_j & -a_j & -a_j+x_j \\ \hline    
\end{array}
\end{equation}

After summing up the components, $F$ is given by 
\begin{equation}
\begin{array}{ll}
F(\mbf{a},\mbf{x}) & = \sum \limits_{ j \in I^{+}_{0 \le x_j < a_j}} x_j 
+ \sum \limits_{ j \in I^{+}_{a_j \le x_j }} a_j
- \sum \limits_{j \in I^{-}_{x_j < a_j}} a_j  
- \sum \limits_{j \in I^{-}_{a_j \le x_j < 0}} x_j, 
\end{array}
\end{equation}
and we have similarly for $G$
\begin{equation}
\begin{array}{ll}
G(\mbf{a},\mbf{x}) & = \sum \limits_{ j \in I^{+}_{x_j < 0 }} (a_j -x_j)
+\sum \limits_{ j \in I^{+}_{0 \le x_j < a_j}} a_j 
+ \sum \limits_{ j \in I^{+}_{a_j \le x_j }} x_j \\
& \quad - \sum \limits_{j \in I^{-}_{x_j < a_j}} x_j  
- \sum \limits_{j \in I^{-}_{a_j \le x_j < 0}} a_j 
- \sum \limits_{j \in I^{-}_{0 \le x_j}} (a_j-x_j). 
\end{array}
\end{equation}

The function $F$ and $G$ are polyhedral, piecewise linear
functions. In the interior of every linear segments they can
be differentiated. The non-differentiable boundaries of the linear segments
are given by those points where at least for one index $j$, $x_j=0$ or
$x_j=a_j$.

\section{Smooth approximation} 

In some application the piecewise differentiability of the general Tanimoto
kernel in both forms,   
(\ref{eq:general_tanimoto_kernel}) or 
(\ref{eq:general_minmax_tanimoto_kernel}), could cause some problems. Here we
present a potential approximation schemes to overcome on this limitation.

\subsection{Smooth approximation of the minimum and maximum functions}   

The smooth approximation of the generalized Tanimoto kernel can be
constructed from the {\it quasi-arithmetic mean}. That concept is also
called as {\it generalized f-mean} or {\it Kolmogorov mean}.  It was
introduced by \cite{AKolmogorov1930}, and later on several extension,
additional properties and applications are published, see
for examples \cite{JBibby1974}, \cite{JAczel1989}. The basic notion
can be derived from {\it Kolmogorov expected value }. Let $f:
\mathbb{R}\subseteq \mathbb{R} \rightarrow \mathbb{R}$ be a
continuous, monotone function, and $X$ is a real valued random variable with
probability distribution function $F$. The  
Kolmogorov expected value is given by  
\begin{equation} 
E_{f}(X)=f^{-1}\left(\: \int \limits_{x \in \mathcal{D}} f(x) dF(x)\right),
\end{equation}
where $f^{-1}$ is the inverse function of $f$. The quasi-arithmetic
mean is a sample based estimation of $E_{f}(X)$. Assume a set
$\mathcal{S}=\{x_1,\dots, x_n\}$ of 
examples  taken from the random variable $X$, then we can 
write up the mean
\begin{equation}
M_{f}(X)=f^{-1}\left( \frac{1}{n} \sum_{i=1}^{n} f(x_i) \right).
\end{equation}
Some basic examples could demonstrate the background of this kind of
concept of the mean value. 
\begin{equation}
\begin{array}{lll}
\renewcommand{\arraystretch}{1.5}
f(x)=x, & \mathcal{S}\subset \mathbb{R}, &  \text{Arithmetic mean}, \\
f(x)=log(x), & \mathcal{S}\subset \mathbb{R}_{+}, &  \text{Geometric mean}, \\
f(x)=\frac{1}{x}, & \mathcal{S}\subset \mathbb{R}_{+}, 
&  \text{Harmonic mean}, \\
f(x)=x^{p}, & \mathcal{S}\subset \mathbb{R}_{+}, 
& \text{Power(generalized) mean}.  
\end{array}
\renewcommand{\arraystretch}{1}
\end{equation}
There are several characteristic properties of these expected value and mean
concepts, \cite{AKolmogorov1930}, \cite{JBibby1974},
\cite{JAczel1989}. For our purpose, the most important property is
that which claims, the 
quasi-arithmetic mean is between the minimum and the maximum value of
the sample. At a proper choice of the function $f$ both the minimum
and the maximum can be approximated by a sample with an arbitrary
small error.   
   
Let $t>0$, and $f(x)=t^{x}$, then $M_{f}(X)=\log_t \left( \frac{1}{n}
  \sum_{i=1}^{n} t^{x_i} \right)$. By taking the following limits we
can approximate the minimum and maximum.      
\begin{equation}
\begin{array}{ll}
\lim_{t\rightarrow \infty} & M_{f}((a,b))  = \max(a,b), \\
\lim_{t\rightarrow -\infty} & M_{f}(a,b) = \min(a,b). \\
\end{array}
\end{equation}
Note the the convergence remains valid if the $\frac{1}{n}$ is
dropped, and only the sum of the $f(x_i)$ elements are considered. 
$M_{f}(X)$ can also be written as $\frac{1}{t}\log \left( \frac{1}{n}
  \sum_{i=1}^{n} e^{t\: x_i} \right)$. 
   
Let $\mbf{a},\mbf{b}\in \mathbb{R}^{n}$ be vectors, the generalized Tanimoto
kernel on these vectors is given by 
\begin{equation}
\begin{array}{@{}ll@{}}
\kappa_{tanimoto}(\mbf{a},\mbf{b}) =\dfrac{
\sum \limits_{\min(a_j,b_j)\ge 0} 
\min(a_j,b_j)   
- \sum \limits_{\max(a_j,b_j)<0} 
\max(a_j,b_j) 
}
{\sum \limits_{\max(a_j,b_j)\ge 0} 
\max(a_j,b_j)   
 -\sum \limits_{\min(a_j,b_j)<0} 
\min(a_j,b_j)
}.   
\end{array}
\end{equation}
It can be reformulated to drop the inequalities from the summation
\begin{equation}
\begin{array}{@{}ll@{}}
\kappa_{tanimoto}(\mbf{a},\mbf{b}) =\dfrac{
\sum_{j=1}^{n} 
\max(\min(a_j,b_j),0)   
- \sum_{j=1}^{n} 
\min(\max(a_j,b_j),0) 
}
{\sum_{j=1}^{n} 
\max(\max(a_j,b_j),0)   
 -\sum_{j=1}^{n} 
\min(\min(a_j,b_j),0)
}.   
\end{array}
\end{equation}
Now we can apply the smooth approximations of the minimums and
maximums appearing in the expression by choosing a sufficiently large
positive value for $t$. 
\begin{equation}
\renewcommand{\arraystretch}{1.5}
\begin{array}{@{}ll@{}}
\max(\min(a_j,b_j),0)
\approx \frac{1}{t}\log(e^{t}+e^{-\log(e^{-ta_j}+e^{-tb_j})}) \\   
\min(\max(a_j,b_j),0) 
\approx -\frac{1}{t}\log(e^{-t}+ e^{-\log(e^{ta_j}+e^{tb_j})}) \\   
\max(\max(a_j,b_j),0)   
\approx \frac{1}{t}\log(e^{t}+e^{\log(e^{ta_j}+e^{tb_j})}) \\   
\min(\min(a_j,b_j),0)
\approx -\frac{1}{t}\log(e^{-t}+ e^{\log(e^{-ta_j}+e^{-tb_j})}) \\   
\end{array}
\renewcommand{\arraystretch}{1}
\end{equation}

Finally the entire approximation of the generalized Tanimoto kernel
has the following form.
\begin{equation}
\begin{array}{@{}ll@{}}
\kappa_{tanimoto}(\mbf{a},\mbf{b}) =\dfrac{
\sum_{j=1}^{n} \log(e^{t}+e^{-\log(e^{-ta_j}+e^{-tb_j})})
+ \sum_{j=1}^{n} \log(e^{-t}+ e^{-\log(e^{ta_j}+e^{tb_j})})
}
{\sum_{j=1}^{n}\log(e^{t}+e^{\log(e^{ta_j}+e^{tb_j})})   
 +\sum_{j=1}^{n} \log(e^{-t}+ e^{\log(e^{-ta_j}+e^{-tb_j})})
}.   
\end{array}
\end{equation}

\section{Experiments}

We evaluate the practical performance of the presented Tanimoto kernels using
multiple prediction tasks.

\subsection{LogP Prediction}

We downloaded a publicly available data
set\footnote{\url{https://cactus.nci.nih.gov/download/nci/ncidb.sdf.gz}} 
containing experimentally determined LogP values for 2671 unique
molecule structures (determined by their SMILES  
representation). The data set was provided by the National Cancer
Institute (NCI \url{https://www.cancer.gov/}). 
Using the SMILES representation of the molecules, we calculated the so
called E-state fingerprints using RDKit  
an Open-source cheminformatics library (\url{https://www.rdkit.org},
version 2019.03.4). E-state fingerprints  
where introduce by \cite{Hall1995} and encode the intrinsic electronic
state (therefore E-state) of 79 predefined 
molecular sub-substructures. The E-state not only depends on the
sub-structure definition it self, but also on 
its atomic neighborhood within the molecule. We can use the E-state
fingerprints in three different representations 
controlling the level of information: As real valued vector of 79
electronic states, as positive valued vector  
counting the predefined sub-structures, or as binary vector only
indicating the presence of the substructure.  

To assess the predictive performance of the generalized Tanimoto
kernel introduced in this work, we optimize three  
different Kernel Ridge Regression (KRR) models, one for each
representation of the E-state fingerprints. Each KRR  
model is evaluated using 3 times repeated 5-fold cross-validation
(CV). The optimal KRR regularization parameter  
is found using nested CV. The KRR prediction performance is shown in
Table~\ref{tab:logp_prediction_performance} and 
compared to a standard prediction model called XLOGP3 developed by
\cite{Cheng2007}. The same CV scheme as for  
the KRR is applied to train and evaluate the XLOGP3 model. It can be
see, that the real-valued fingerprints lead  
to the best KRR model. Its performance is very close to XLOGP3 model,
allowing competitive LogP predictions using 
the generalized Tanimoto kernel. An advantage of the kernel is, that
no further hyper-parameter needs to be optimized. 

\begin{table}[t]
    \centering
    \begin{tabular}{lcccc}
        \toprule
        {\bf Model} & {\bf MSE} & {\bf R$^2$} & {\bf Pearson} & {\bf Spearman} \\ \midrule
        KRR + Binary & 1.085 (0.070) & 0.678 (0.023) & 0.825 (0.013) & 0.797 (0.021) \\
        KRR + Count & 0.278 (0.032) & 0.917 (0.012) & 0.958 (0.006) & 0.951 (0.007) \\
        KRR + Real & 0.228 (0.020) & 0.932 (0.009) & 0.966 (0.005) & 0.960 (0.005) \\ \cmidrule(lr){1-5}
        XLOGP3 & 0.220 (0.021) & 0.935 (0.007) & 0.967 (0.004) & 0.961 (0.005) \\
        \bottomrule
    \end{tabular}
    \caption{LogP prediction performance using different models. The results 
    in the first three rows where achieved using Kernel Ridge Regression (KRR) with 
    different molecule representations and generalized Tanimoto kernel. The last row contains
    the results of the XLOGP3 LogP prediction model developed by \cite{Cheng2007}. All models
    where evaluated using 3-times 5-fold cross-validation and the average (standard deviation)
    performance is reported.}
    \label{tab:logp_prediction_performance}
\end{table}


\section*{Acknowledgments}

%

This work has been supported by the Academy of Finland grant 310107
(MACOME - Machine Learning for Computational Metabolomics).

\bibliography{tanimoto_bibliography.bib}

\begin{thebibliography}{10}

\bibitem{Shawe-Taylor2004}
J.~Shawe-Taylor and N.~Cristianini.
\newblock {\em Kernel Methods for Pattern Analysis}.
\newblock Kernel Methods for Pattern Analysis. Cambridge University Press,
  2004.

\bibitem{doi:10.1111/j.1469-8137.1912.tb05611.x}
Paul Jaccard.
\newblock The distribution of the flora in the alpine zone.1.
\newblock {\em New Phytologist}, 11(2):37--50, 1912.

\bibitem{Ioffe36928}
Sergey Ioffe.
\newblock Improved consistent sampling, weighted minhash and l1 sketching.
\newblock In {\em ICDM}, 2010.

\bibitem{nla.cat-vn1717218}
T.~T. Tanimoto.
\newblock {\em An elementary mathematical theory of classification and
  prediction by T.T. Tanimoto}.
\newblock International Business Machines Corporation New York, 1958.

\bibitem{Hall1995}
Lowell~H. Hall and Lemont~B. Kier.
\newblock Electrotopological state indices for atom types: A novel combination
  of electronic, topological, and valence state information.
\newblock {\em Journal of Chemical Information and Computer Sciences},
  35(6):1039--1045, 1995.

\bibitem{butina_ci9803381}
Darko Butina.
\newblock Unsupervised data base clustering based on daylight's fingerprint and
  tanimoto similarity: A fast and automated way to cluster small and large data
  sets.
\newblock {\em Journal of Chemical Information and Computer Sciences},
  39/4:747--750, 1999.

\bibitem{doi:10.1002/9780470116449.ch6}
Ovidiu Ivanciuc.
\newblock {\em Applications of Support Vector Machines in Chemistry},
  chapter~6, pages 291--400.
\newblock John Wiley \& Sons, Ltd, 2007.

\bibitem{RALAIVOLA20051093}
Liva Ralaivola, Sanjay~J. Swamidass, Hiroto Saigo, and Pierre Baldi.
\newblock Graph kernels for chemical informatics.
\newblock {\em Neural Networks}, 18(8):1093 -- 1110, 2005.
\newblock Neural Networks and Kernel Methods for Structured Domains.

\bibitem{Geppert_ci900419k}
Hanna Geppert, Martin Vogt, and Jürgen Bajorath.
\newblock Current trends in ligand-based virtual screening: Molecular
  representations, data mining methods, new application areas, and performance
  evaluation.
\newblock {\em Journal of Chemical Information and Modeling}, 50/2:205--216,
  2010.

\bibitem{10.5555/1893126}
Huma Lodhi and Yoshihiro Yamanishi.
\newblock {\em Chemoinformatics and Advanced Machine Learning Perspectives:
  Complex Computational Methods and Collaborative Techniques}.
\newblock IGI Global, USA, 1st edition, 2010.

\bibitem{LAVECCHIA2015318}
Antonio Lavecchia.
\newblock Machine-learning approaches in drug discovery: methods and
  applications.
\newblock {\em Drug Discovery Today}, 20(3):318 -- 331, 2015.

\bibitem{Bach2018}
Eric Bach, Sandor Szedmak, C\'{e}line Brouard, Sebastian B\"{o}cker, and Juho
  Rousu.
\newblock Liquid-chromatography retention order prediction for metabolite
  identification.
\newblock {\em Bioinformatics}, 34(17):i875--i883, 2018.

\bibitem{Rockafellar1970}
R.~T. Rockafellar.
\newblock {\em Convex Analysis}, volume Princeton Math. Series 28.
\newblock Princeton University Press, 1970.

\bibitem{Rockafellar1997}
R.~T. Rockafellar and R.J.B. Wets.
\newblock {\em Variational Analysis}.
\newblock Springer, 1997.

\bibitem{AKolmogorov1930}
Andrey Kolmogorov.
\newblock On the notion of mean.
\newblock In {\em Mathematics and Mechanics}, pages 144--146. Kluwer 1991,
  1930.

\bibitem{JBibby1974}
John Bibby.
\newblock Axiomatisations of the average and a further generalisation of
  monotonic sequences.
\newblock {\em Glasgow Mathematical Journal}, 15:63--65, 1974.

\bibitem{JAczel1989}
J.~Acz\'{e}l and J.~G. Dhombres.
\newblock Functional equations in several variables. with applications to
  mathematics, information theory and to the natural and social sciences.
\newblock In {\em Encyclopedia of Mathematics and its Applications}, volume~31.
  Cambridge: Cambridge Univ. Press, 1989.

\bibitem{Cheng2007}
Tiejun Cheng, Yuan Zhao, Xun Li, Fu~Lin, Yong Xu, Xinglong Zhang, Yan Li,
  Renxiao Wang, and Luhua Lai.
\newblock Computation of octanol-water partition coefficients by guiding an
  additive model with knowledge.
\newblock {\em Journal of Chemical Information and Modeling}, 47(6):2140--2148,
  2007.
\newblock PMID: 17985865.

\end{thebibliography}
\bibliographystyle{unsrt}

\end{document}